\documentclass[11pt, draftclsnofoot, onecolumn]{IEEEtran}
\usepackage{amsfonts}%
\usepackage{amssymb}%
\usepackage[dvips]{graphicx}%
\usepackage{amsmath}%
\usepackage[usenames]{color}
\usepackage{algorithmic}
\newtheorem{theorem}{Theorem}
\newtheorem{corollary}{Corollary}

\newtheorem{assumption}{Assumption}

\newtheorem{example}{Example}

\newtheorem{remark}{Remark}
\newcommand{\comment}[1]{}

\DeclareMathOperator*{\argmax}{arg\,max}

\ifodd 0
\newcommand{\cem}[1]{{\color{blue}#1}}
\newcommand{\cm}[1]{{\color{green}#1}}
\else
\newcommand{\cem}[1]{#1}
\newcommand{\cm}[1]{#1}
\fi

\ifodd 0
\newcommand{\rev}[1]{{\color{blue}#1}} 
\newcommand{\com}[1]{\textbf{\color{red}(COMMENT: #1)}} 
\newcommand{\clar}[1]{\textbf{\color{green}(NEED CLARIFICATION: #1)}}

\else
\newcommand{\rev}[1]{#1}
\newcommand{\com}[1]{}
\newcommand{\clar}[1]{}
\fi

\begin{document}
\title{Online Learning in a Contract Selection Problem}
\author{Cem Tekin, Mingyan Liu
\thanks{C. Tekin is with the Dept. of Electrical Engineering, University of California, Los Angeles, CA, 90095, 
M. Liu is with the Dept. of Electrical Engineering and Computer Science, University of Michigan, Ann Arbor, MI 48105, cmtkn@ucla.edu, mingyan@eecs.umich.edu.}
\thanks{This work is supported by NSF grant CIF-0910765 and ARO grant
W911NF-11-1-0532.}}

\maketitle

\begin{abstract}
In an online contract selection problem there is a seller which offers a set of contracts to sequentially arriving buyers whose types are drawn from an unknown distribution. If there exists a profitable contract for the buyer in the offered set, i.e., a contract with payoff higher than the payoff of not accepting any contracts, the buyer chooses the contract that maximizes its payoff.
In this paper we consider the online contract selection problem to maximize the sellers profit.  Assuming that a structural property called {\em ordered preferences} holds for the buyer's payoff function, 
we propose online learning algorithms that have sub-linear regret with respect to the best set of contracts given the distribution over the buyer's type. \cem{This problem has many applications including spectrum contracts, wireless service provider data plans and recommendation systems.}
\end{abstract}
\section{Introduction} \label{contract:sec:intro}

In an online contract selection problem there is a seller who offers a bundle of contracts to buyers arriving sequentially over time. The goal of the seller is to maximize its total expected profit up to the final time $T$, by learning the best bundle of contracts to offer. However, the seller does not know the best bundle of contracts beforehand because initially it does not know the preferences of the buyers. 

Assuming that the buyers' preferences change stochastically over time, our goal in this paper is to design learning algorithms for the seller to maximize its expected profit. 
Specifically, we assume that the preferences of a buyer depends on its {\em type}, and is given by a payoff function depending on the type of the buyer. 
The type of the buyer at time step $t$ is drawn from a distribution not known by the seller, independently from the other time steps. Obviously, the best bundle of contracts (which maximizes the sellers expected profit) depends on the distribution of the buyers' type and the preferences of the buyers.

We assume that the seller can choose what to offer from a continuum of contracts, but it should choose a finite number of contracts to offer simultaneously. 
We show that if the buyers' payoff function has a special property which we call the ordered preferences property, then there exists learning algorithms for the seller by which the seller can estimate the type distribution of the buyers by offering contracts and observing the contracts that are accepted by the buyers. Then, the seller can compute the expected payoff of different bundles of contracts using the estimated type distribution.

The online contract selection problem can be viewed as a combinatorial multi-armed bandit problem, where each arm is a vector (bundle) of contracts, and each component of the vector can be chosen from an interval of the real line. Two aspects that make this problem harder than the classical multi-armed bandit problem are: (i) uncountable number of contracts; (ii) exponential number of arms in the number of contracts. We can overcome (i) by offering bundles of sufficiently {\em closely spaced} contracts to form an estimate of the distribution of buyer's type, and (ii) by writing the expected payoff of an arm as a function of the expected payoffs of the contracts in that arm. Exploiting the structure of the problem, we prove sublinear regret bounds which scale linearly in the dimension $m$.

The online learning problem we consider in this \rev{paper} involves large strategy sets, combinatorial and contextual elements. Problems with a continuum of arms are considered in \cite{agrawal1995continuum, kleinberg2004nearly, cope2009regret, auer2007improved}, where sub-linear regret results are derived. Several combinatorial bandit problems are studied in \cite{gai2012combinatorial, gai2012restless}, and problems involving stochastic linear optimization are considered in \cite{bartlett2008high, dani2008stochastic}. Another line of work \cite{kleinberg2008multi} generalized the continuum armed bandits to bandits on metric spaces. In this setting, the difference between the expected payoffs of a pair of arms is related to the distance between the arms via a {\em similarity} metric. Contextual bandits, in which context information is provided to the algorithm at each round is studied in \cite{langford2007epoch, slivkins2009contextual}. The goal there is to learn the best arm, given the context.

\cem{Many applications can be modeled using our online contract selection framework. For instance, the bundles of contracts can be viewed as wireless service provider data plans from which the users make selections based on their needs, or they can be viewed as recommendations in a recommender system such as airline tickets, hotels, rental cars, etc., where each recommendation has a rating that encodes its price and quality.}

The organization of the rest of the paper is as follows. In Section \ref{contract:sec:probform}, we define the contract selection problems, the ordered preferences property, and provide three applications of the problem. We propose a contract learning algorithm with variable number of simultaneous offers at each time step in Section \ref{contract:sec:alg1}, and analyze its performance in Section \ref{contract:sec:analysis1}. Then, we consider a variant of this algorithm with fixed number of offers in Section \ref{contract:sec:alg2}. Finally, we discuss the similarities and the differences between our work and the related work in Section \ref{contract:sec:discuss}.

 \section{Problem Formulation and Preliminaries} \label{contract:sec:probform}

In an online contract selection problem there is a seller who offers a bundle of $m \in \{1,2,\ldots\}$ contracts $\boldsymbol{x} \in {\cal X}_m$, where
\begin{align*}
{\cal X}_m := \left\{ (x_1, x_2,\ldots, x_m), \textrm{ such that } x_i \in (0,1], \forall i \in \{1,2,\ldots,m\}, x_i \leq x_{i+1}   \right\},
\end{align*}
to buyers arriving sequentially at time steps $t=1,2, \ldots, T$, where $T$ is the time horizon. Let $\boldsymbol{x}(t)$ be the bundle offered by the seller at time $t$. The buyer can accept a single contract $y \in \boldsymbol{x}(t)$ and pay $y$ to the seller, or it can reject all of the offered contracts and pay $0$ to the seller. 
Profit of the seller by time $T$ is
\begin{align*}
\sum_{t=1}^T (u(t) - c(t)) ,
\end{align*}
where $u(t)$ represents the revenue/payoff of the seller at time $t$ and $c(t)$ is any cost associated with offering the contracts at time $t$.
We have $u(t) =x$ if contract $x$ is accepted by the buyer at time $t$, $u(t) =0$ if none of the offered contracts at time $t$ is accepted by the buyer at time $t$.

The buyer who arrives at time $t$ has type $\theta_t$ which encodes its preferences into a payoff function. 
At each time step, the type of the buyer present at that time step is drawn according to the probability density function $f(\theta)$ on $[0,1]$ independently from the other time steps. We assume that buyer's type density is bounded, i.e.
\begin{align*}
f_{\max} := \sup_{\theta \in [0,1]} f(\theta) < \infty.
\end{align*}
Neither $\theta_t$ nor $f(\theta)$ is known by the seller at any time. 
Therefore, in order maximize its profit, the seller should learn the best set of contracts over time. The expected profit of the seller over time horizon $T$ is given by
\begin{align*}
E \left[\sum_{t=1}^T u(t) - c(t) \right],
\end{align*}
where the expectation is taken with respect to buyer's type distribution $f(\theta)$ and the seller's contract offering strategy. Our goal in this paper is to develop online learning algorithms for the seller to maximize its expected profit over time horizon $T$.

Let $U_b(x,\theta) : [0,1] \times [0,1] \rightarrow \mathbb{R}$ represent the payoff function of type $\theta$ buyer, which is a function of the contract accepted by the buyer. We assume that the seller knows $U_b$. For example, when the contracts represent data plans of wireless service providers, the service provider can know the worth of a 2 GB contract to a buyer who only needs 1 GB each month. For instance, the amount of payment for the 2 GB contract that exceeds the payment for a 1 GB can represent the loss of the buyer. Similarly, a 500 MB contract to a buyer who needs 1 GB a month can have a cost equal to the 500 MB shortage in data. Of course there should be a way to relate the monetary loss with the data loss, which can be captured by coefficients multiplying these two. These coefficients can also be known by the seller by analyzing previous buyer data. 

Based on its payoff function, the buyer either selects a contract from the offered bundle or it may reject all of the contracts in the bundle. If $\boldsymbol{x}= (x_1, x_2,\ldots, x_m)$ is offered to a type $\theta$ buyer, it will accept a contract randomly from the set 
\begin{align*}
\argmax_{ x \in \{0, x_1, \ldots, x_m \} } U_b(x, \theta),
\end{align*}
where $x=0$ implies that the buyer does not accept any of the offered contracts.
Since the seller knows the buyer's payoff function $U_b(x,\theta)$, for a given bundle of contracts $\boldsymbol{x}= (x_1, x_2,\ldots, x_m)$, it can compute which contracts will be accepted as a function of the buyer's type. For $y \in \boldsymbol{x}$, let $I_{y}(\boldsymbol{x})$ be the acceptance region of contract $y$, which is the values of $\theta$ for which contact $y$ will be accepted from the bundle $\boldsymbol{x}$ \cm{(usually an interval of the real line)}. \cm{For two intervals of the real line $I_1$ and $I_2$, $I_1 < I_2$ means that the rightmost point of $I_1$ is less than or equal to the leftmost point of $I_2$.}
We assume that the buyers payoff function induces {\em ordered preferences}, which means that for a bundle of contracts $(x_1, \ldots, x_m)$, the values of $\theta$ for which $x_i$ is accepted only depends on $x_{i-1}$, $x_i$ and $x_{i+1}$, and 
\begin{align*}
I_{x_{i-1}}(\boldsymbol{x}) < I_{x_{i}}(\boldsymbol{x}) < I_{x_{i+1}}(\boldsymbol{x}),
\end{align*}
for all $i \in \{1,2,\ldots,m-1 \}$, which means that the acceptance regions are ordered.
\begin{assumption}\label{ass:contract:ordered} \textbf{Ordered Preferences.}
$U_b(x,\theta)$ induces ordered preferences which means that 
for any $\boldsymbol{x} \in {\cal X}$, $I_{x_i}(\boldsymbol{x}) = (g(x_{i-1}, x_i), g(x_i, x_{i+1}) ]$. The function $g$ is such that $g(x,y) < g(y,z)$ for $x<y<z$, and $g$ is H{\"o}lder continuous with constant $L$ and exponent $\alpha$, i.e.,
\begin{align*}
|g(x_1, x_2) - g(y_1, y_2)| \leq L \sqrt{(|x_1 - y_1|^2 + |x_2 - y_2|^2)}^\alpha .
\end{align*}
\end{assumption}
Although the assumption on $U_b(x,\theta)$ is implicit, it is satisfied by many common payoff functions. Below we provide several examples. For notational convenience for any bundle of contracts $(x_1, x_2, \ldots, x_m)$, let $x_0 = 0$, $x_{m+1}=1$ and $g(x_m, 1) =1$.

\begin{example}{\bf Wireless Data Plan Contract.}\label{contract:example:1}
\cm{A wireless user's goal is to have a certain video/audio quality and download/upload quota. For a wireless user with type $\theta$, it is intuitive to assume that $(x-\theta)^+$ corresponds to loss in accepting a contract which offers data less than the demand, while $( \theta - x )^+$ corresponds to loss in accepting a contract $x$ which offers data more than the demand but have a higher price than the price of the demanded data service. Tradeoff between these two is captured by coefficients that relate the buyers weighting of these losses. Therefore we assume that the buyer's payoff function is given by}
\begin{align*}
U_b(x, \theta) := h(a (x - \theta)^+ + b ( \theta - x )^+),
\end{align*}
where 
\begin{align*}
(x - y)^+ = \max\{0, x-y \},
\end{align*}
and
$h: \mathbb{R_+} \rightarrow \mathbb{R}$ is a decreasing function. For data plan contracts,  
For this payoff function, the accepted contract from any bundle $(x_1, x_2, \ldots, x_m)$ of contracts is given as a function of the buyer's type in Figure \ref{contract:fig:acceptance:1}. It is easy to check that the boundaries of the acceptance regions are 
\begin{align*}
g(x_{i-1}, x_i) = \frac{b x_{i-1}+ a x_i}{a+b}, ~~ \forall i = 1,2,\ldots,m.
\end{align*}
Since 
\begin{align*}
| g(x_1, x_2) - g(y_1, y_2)| &= \left| \frac{b (x_1 - y_1)}{a+b} + \frac{a(x_2 - y_2)}{a+b} \right| \\
& \leq \max \{ |x_1 - y_1| , |x_2 - y_2|  \} \\
&\leq \sqrt{|x_1 - y_1|^2 + |x_2 - y_2|^2 },
\end{align*}
Assumption \ref{ass:contract:ordered} holds for this buyer payoff function with $L=1$ and $\alpha = 1$. 
\end{example}
\begin{figure}
\begin{center}
\includegraphics[width=0.8\columnwidth]{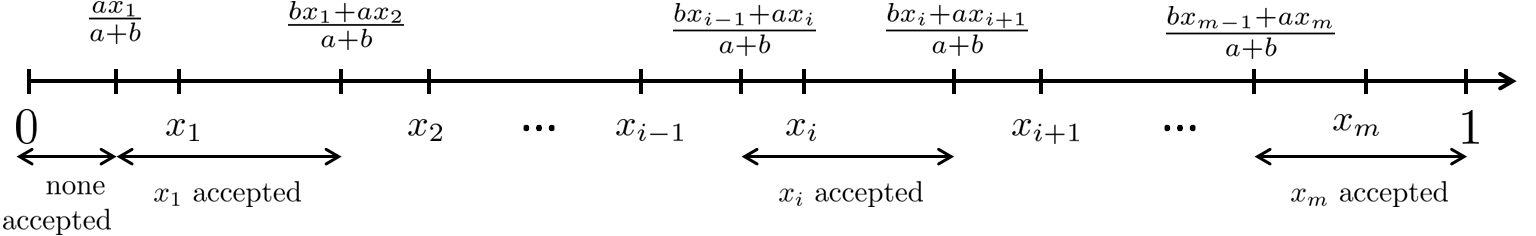}
\caption{acceptance region of bundle $(x_1, \ldots, x_m)$ for $U_b(x, \theta) = h(a (x - \theta)^+ + b ( \theta - x )^+)$} 
\label{contract:fig:acceptance:1}
\end{center}
\end{figure}

\begin{example} {\bf Secondary Spectrum Contract.}\label{contract:example:2}
Consider a secondary spectrum market, where the service provider leases excess spectrum to secondary users. For simplicity, assume that the service provider always have a unit bandwidth available. In general, due to the primary user activity the bandwidth available for leasing at time $t$ is $B_t \in [0,1]$, however, all our results in this paper will hold for dynamically changing available bandwidth, provided that the seller pays a penalty to the buyers for any bandwidth it offers but cannot guarantee to a buyer. By this way, the seller can still learn the buyer's type distribution by offering a bundle of contracts $\boldsymbol{x}$ for which there is some $x_i > B_t$.
The buyer's payoff function in this case is 
\begin{align*}
U_b(x, \theta) :=- a (\theta - x)^+ - x,
\end{align*}
where $x$ is the amount of money that the buyer pays to the seller by accepting contract $x$ and $a>1$ is a coefficient that relates the tradeoff between the loss in data and monetary loss. For this payoff function it can be shown that the acceptance region boundaries are
\begin{align*}
g(x_{i-1}, x_i) = \frac{ (a-1) x_{i-1}+ x_i}{a}, ~~ \forall i = 1,2,\ldots,m,
\end{align*}
and Assumption \ref{ass:contract:ordered} holds with $L=1$, $\alpha=1$.
\end{example}
\cem{
\begin{example} {\bf Recommendation System.}\label{contract:example:3}
Consider a recommendation system where the recommender makes $m$ recommendations to each arriving user. For instance these recommendations can be flights, hotels, etc. Each recommendation has a rating $x \in [0,1]$ which reflects some weighted average of quality and price. Based on its (unknown) type $\theta$ (budget, preferences, etc.) a user either chooses one of the recommendations or it does not accept any. For example, type of the user may represent its preferred rating, and the user may only accept a recommendation if its rating is within $(\theta - \epsilon, \theta+\epsilon)$ for some small $\epsilon>0$. Note that this user preference satisfies the ordered preferences property given in Assumption \ref{ass:contract:ordered} where 
\begin{align*}
I_{x_i}(\boldsymbol{x}) = I(x_i \in (\theta - \epsilon, \theta+\epsilon) \textrm{ and } x_i \textrm{ closest in } \boldsymbol{x} \textrm{ to } \theta)\footnote{If there exists more than one recommendation in $(\theta - \epsilon, \theta+\epsilon)$, the user chooses the closest recommendation to its type. If the distance is the same, then one of the recommendation is chosen arbitrarily.}
\end{align*}
Note that although we do not require the H{\"o}lder condition to hold here, our analysis throughout the paper hold for this type of user preference as well. Whenever a user chooses a recommendation, the recommender obtains reward $1$. The goal of the recommender is to maximize the number of sales. 
\end{example}
}

\begin{figure}
\begin{center}
\includegraphics[width=0.8\columnwidth]{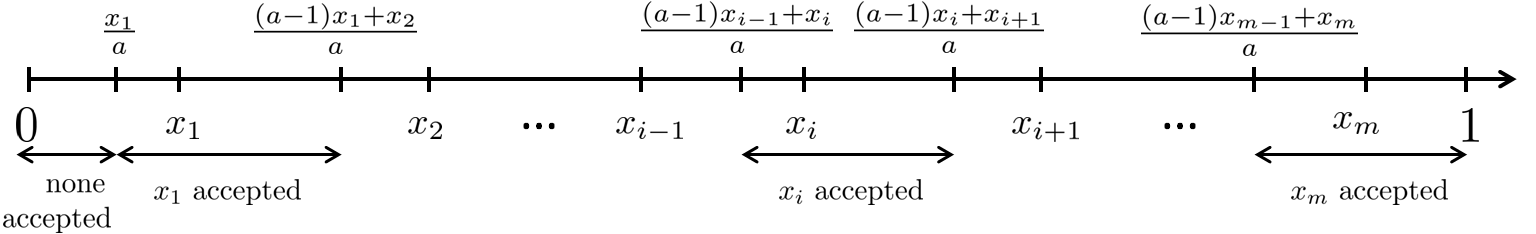}
\caption{acceptance region of bundle $(x_1, \ldots, x_m)$ for $U_b(x, \theta) = - a (\theta - x)^+ - x$} 
\label{contract:fig:acceptance:2}
\end{center}
\end{figure}
\cm{Above examples show that our contract selection framework has a large set of applications.} By Assumption \ref{ass:contract:ordered}, the expected payoff of a bundle of contracts $\boldsymbol{x} \in {\cal X}_m$ to the seller is 
\begin{align*}
U(\boldsymbol{x}) &= x_1 P( g(0,x_1) < \theta \leq g(x_1, x_2)) 
+ x_2 P( g(x_1,x_2) < \theta \leq g(x_2, x_3)) \\
&+ \ldots
+ x_m P(g(x_{m-1}, x_m) < \theta).
\end{align*}

Note that the seller's problem would be solved if it knew $f(\theta)$, since it could compute the best bundle of $m$ contracts, i.e.,
\begin{align}
\argmax_{\boldsymbol{x} \in {\cal X}_m} U(\boldsymbol{x}). \label{contract:eqn:optimal}
\end{align}

\begin{remark}
We do not require that the maximizer of (\ref{contract:eqn:optimal}) is a bundle of $m$ distinct contracts. Note that by definition of the set ${\cal X}_m$, the maximizer of (\ref{contract:eqn:optimal}) may be a bundle $(x_1, \ldots, x_m)$ for which $x_i = x_{i+1}$ for some $i \in \{1,2,\ldots, m-1 \}$. This is equivalent to offering $m-1$ contracts $(x_1, \ldots, x_{i-1}, x_i, x_{i+2}, \ldots, x_m)$. Indeed, our results hold when the seller's goal is to learn the best bundle of contracts that have at most $m$ contracts in it.
\end{remark}

The key idea behind the learning algorithms we design for the seller in the subsequent sections is to form estimates of the buyer's \cm{type} distribution by offering different sets of contracts. Each algorithm consists of exploration and exploitation phases. Although we stated that the seller offers $m$ contracts at each time step, in our first algorithm $m$ will vary over time \cm{(two different values for $m$; one for exploration, one for exploitation)}, so we denote it by $m(t)$. In our second algorithm $m$ will be fixed throughout the time horizon $T$. 

We also assume that the cost of offering $m$ contracts at the same time is given by $c(m)$, which increases with $m$. \cm{For example, in a recommendation system, the buyer can get confused when there is a huge list of recommendations, for wireless service providers the regulatory costs increase with the number of different plans offered, etc.} The seller's objective of maximizing the profit over time horizon $T$ is equivalent to minimizing the regret which is given by 
\begin{align}
R^{\textrm{alg}}(T) = T (U(\boldsymbol{x}^*) - c(m)) 
- E^{\textrm{alg}} \left[ \sum_{t=1}^T r(\boldsymbol{x}(t)) - c(m(t))\right], \label{eqn:prob_regret1}
\end{align}
where
\begin{align}
 \boldsymbol{x}^* \in \argmax_{\boldsymbol{x} \in {\cal X}_m} U(\boldsymbol{x}), \label{eqn:prob_optimal}
\end{align}
is the optimal set of $m$ contracts, $r(\boldsymbol{x}(t))$ is the payoff of the seller from the bundle offered at time $t$, and $\textrm{``alg''}$ is the learning algorithm used by the seller. We will drop the superscript $\textrm{``alg''}$ when the algorithm used by the seller is clear from the context. Note that for any algorithm with sublinear regret, the time averaged expected profit will converge to $U(\boldsymbol{x}^*) - c(m)$. 

 \section{A Learning Algorithm with Variable Number of Offers} \label{contract:sec:alg1}

In this section we present a learning algorithm which \cm{sequences time steps into exploration and exploitation steps, and uses the exploration steps to learn about the buyer's type distribution while using the exploitation steps to offer best bundle of contracts.}
The algorithm is called {\em type learning with variable number of offers} (TLVO), whose pseudocode is given in Figure \ref{fig:contract1}.

\rev{Instead of searching for the best bundle of contracts in ${\cal X}_m$ which is uncountable, the algorithm searches for the best bundle of contracts in the finite set 
\begin{align*}
{\cal L}_{m,T} := \left\{ \boldsymbol{x} = (x_1, \ldots x_m) : x_i \leq x_{i+1} \textrm{ and } x_i \in {\cal K}_T, \forall i \in \{1,\ldots, m\} \right\},
\end{align*}
where
\begin{align*}
{\cal K}_T := \left\{ \frac{1}{n_T}, \frac{2}{n_T}, \ldots, \frac{n_T -1}{n_T} \right\}.
\end{align*}
} 
Here $n_T$ is a non-decreasing function of the time horizon $T$. Since the best bundle in ${\cal L}_{m,T}$ might have an expected reward smaller than the expected reward of the best bundle in ${\cal X}_m$, in order to bound the regret due to this difference sublinearly over time, $n_T$ should be adjusted according to the time horizon.

\begin{figure}[h!]
\begin{center}
\fbox {
\begin{minipage}{0.9\columnwidth}
\flushleft{Type Learning with Variable Number of Offers (TLVO)}
\begin{algorithmic}[1]
\STATE {Parameters: $m$, $T$, $z(t), 1 \leq t \leq T$, $n_T$, ${\cal K}_T$, ${\cal L}_{m,T}$.}
\STATE {Initialize: set $t=1$, $N=0$, $\mu_i=0, N_i=0, \forall i \in {\cal K}_T$.}
\WHILE {$t\geq 1$}
\IF{$N < z(t)$}
\STATE{EXPLORE}
\STATE{Offer all contracts in ${\cal K}_T$ simultaneously.}
\IF{Any contract $x \in {\cal K}_T$ is accepted by the buyer}
\STATE{Get reward $x$. Find $k \in \{ 1,2,\ldots, n_T-1 \}$ such that $k/n_T =x$.}
\STATE{$++N_k$.}
\ENDIF
\STATE{$++N$.}
\ELSE
\STATE{EXPLOIT}
\STATE{$\mu_i = N_i/N, \forall i \in \{ 1,2,\ldots, n_T-1 \}$.}
\STATE{Offer bundle $\boldsymbol{x} = (x_1,\ldots, x_m)$, which is a solution to (\ref{alg1:eqn:1}) based on $\mu_i$'s.}
\STATE{If some $x \in \boldsymbol{x}$ is accepted, get reward $x$.}
\ENDIF
\STATE{$++t$.}
\ENDWHILE
\end{algorithmic}
\end{minipage}
} \caption{pseudocode of TLVO} \label{fig:contract1}
\end{center}
\end{figure}

Exploration and exploitation steps are sequenced in a deterministic way. This sequencing is provided by a {\em control function} $z(t)$ which is a parameter of the learning algorithm. Let $N(t)$ be the number of explorations up to time $t$. If $N(t) < z(t)$, time $t$ will be an exploration step. Otherwise time $t$ will be an exploitation step. While $z(t)$ can be any sublinearly increasing function, we will optimize over $z(t)$ in our analysis.

In an exploration step, TLVO estimates the distribution of buyer's type by simultaneously offering the set of $n_T - 1$ uniformly spaced contracts in ${\cal K}_T$.
Based on the accepted contract at time $t$, the seller learns the part of the type space that the buyers type at $t$ lies in, and uses this to form sample mean estimates of the distribution of the buyer's type. \cm{An example is shown in Figure \ref{contract:fig:acceptance:3}.} We simply call the contract $i/n_T \in {\cal K}_T$ as the $i$th contract. Let $\theta$ be the unknown type of the buyer at some exploration step. If $i$th contract is accepted by the buyer, then the seller knows that
\begin{figure}
\begin{center}
\includegraphics[width=0.8\columnwidth]{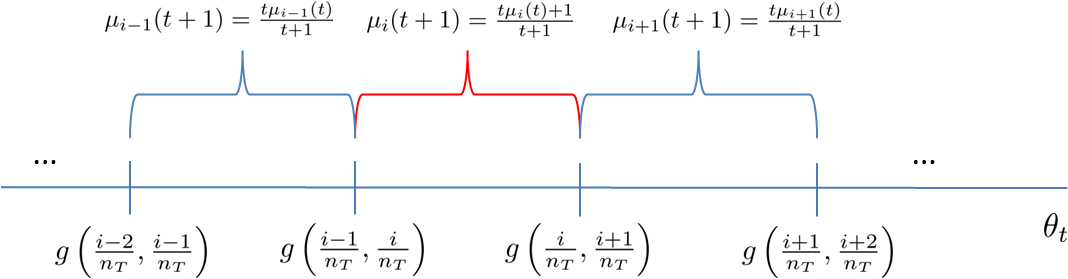}
\caption{an illustration of the sample mean update of buyer's type distribution when contract $i$ is accepted at time $t$} 
\label{contract:fig:acceptance:3}
\end{center}
\end{figure}
\begin{align*}
g \left(\frac{i-1}{n_T}, \frac{i}{n_T} \right) < \theta \leq g \left( \frac{i}{n_T}, \frac{i+1}{n_T} \right).
\end{align*}
Let $N_i(t)$ be the number of times contract $i$ is accepted in an exploration step up to $t$. Then the sample mean estimate of
\begin{align*}
P \left( g \left(\frac{i-1}{n_T}, \frac{i}{n_T} \right) < \theta \leq g \left(\frac{i}{n_T}, \frac{i+1}{n_T} \right)  \right), 
\end{align*}
is given by
\begin{align*}
\mu_i(t) := \frac{N_i(t)}{N(t)}.
\end{align*}

%

%
In an exploitation step, TLVO offers a bundle of $m$ contracts chosen from ${\cal L}_{m,T}$, which maximizes the seller's estimated expected payoff. For constants $\theta_l$ and $\theta_u$, let $\hat{P}_t( \theta_l < \theta \leq \theta_u )$ be the estimate of $P(\theta_l < \theta \leq \theta_u)$ at time $t$. TLVO computes this estimate based on the estimates $\mu_i(t)$ in the following way:
\begin{align*}
\hat{P}_t( \theta_l < \theta \leq \theta_u ) = \sum_{i = j_{-}(\theta_l)} ^{j_{+}(\theta_u)} \mu_i(t),
\end{align*}
where
\begin{align*}
j_{-}(\theta_l) = \min \left\{i \in \{1,\ldots, n_T-1 \} \textrm{ such that } 
g \left( \frac{i-1}{n_T}, \frac{i}{n_T} \right) \geq \theta_l   \right\}, 
\end{align*}
and
\begin{align*}
j_{+}(\theta_u) = \min \left\{i \in \{1,\ldots, n_T-1 \}  \textrm{ such that } 
g \left( \frac{i}{n_T}, \frac{i+1}{n_T}  \right) \geq \theta_u   \right\}. 
\end{align*}

We can write $x_i \in {\cal L}_{m,T}$ as $k_i/n_T$ for some $k_i \in \{1, 2, \ldots, n_T-1 \}$. If time $t$ is an exploitation step, TLVO computes the estimated best bundle of contracts $\boldsymbol{x}(t)$ by solving the following optimization problem.
\begin{align}
\boldsymbol{x}(t) &= \argmax_{\boldsymbol{x} \in {\cal L}_{m,T}} \hat{U}_t(\boldsymbol{x}), \label{alg1:eqn:1}
\end{align}
where 
\begin{align*}
\hat{U}_t(\boldsymbol{x}) &:= 
x_1 \hat{P}_t \left( g(0, x_1) < \theta \leq g(x_1, x_2) \right)
+ x_2 \hat{P}_t \left( g(x_1, x_2) < \theta \leq g(x_2, x_3) \right) \notag \\
&+ \ldots 
+ x_m \hat{P}_t \left( g(x_{m-1} , x_m) < \theta \right).
\end{align*}
%
%
Note that there might be more than one maximizer to (\ref{alg1:eqn:1}). In such a case, TLVO arbitrarily chooses one of the maximizer bundles. Maximization in (\ref{alg1:eqn:1}) is a combinatorial optimization problem. In general solution to such a problem is NP-hard.
We assume that the solution is provided to the algorithm by an oracle. This is a common assumption in online learning literature, for example used in \cite{dani2008stochastic}. Therefore, we do not consider the computational complexity of this operation. Although we do not provide a computationally efficient solution for (\ref{alg1:eqn:1}), there exists computationally efficient methods for some special cases.
We discuss more on this in Section \ref{contract:sec:discuss}.

We analyze the regret of TLVO in the next section.

 \section{Analysis of the Regret of TLVO} \label{contract:sec:analysis1}

In this section we upper bound the regret of TLVO. Let
\begin{align*}
S = \{\boldsymbol{x} \in {\cal L}_{m,T}: |U(\boldsymbol{x}^*) - U(\boldsymbol{x})| < \beta n_T^{-\alpha}  \},
\end{align*}
be the set of near-optimal bundles of contracts where $\alpha$ is the H{\"o}lder exponent in Assumption \ref{ass:contract:ordered},
and $\beta = 5 m f_{\max} L 2^{\alpha/2}$ is a constant where $L$ is the H{\"o}lder constant in Assumption \ref{ass:contract:ordered}. \cm{Denote the complement of $S$ on ${\cal L}_{m,T}$, i.e., ${\cal L}_{m,T} - S$,  by $S^c$.}
Let $T_{\boldsymbol{x}}(t)$ be the number of times $\boldsymbol{x} \in {\cal L}_{m,T}$ is offered at exploitation steps by time $t$.
For TLVO, regret given in (\ref{eqn:prob_regret1}) is upper bounded by
\begin{align}
R(T) &\leq \sum_{\boldsymbol{x} \in S} (U(\boldsymbol{x}^*) - U(\boldsymbol{x})) E \left[ T_{\boldsymbol{x}}(T) \right] \notag \\
&+ \sum_{\boldsymbol{x} \in S^c} (U(\boldsymbol{x}^*) - U(\boldsymbol{x})) E \left[ T_{\boldsymbol{x}}(T) \right] \notag \\
&+ N(T) (U(\boldsymbol{x}^*) + c(n_T) - c(m)), \label{r1:eqn:regret1}
\end{align}
by assuming zero worst-case payoff in exploration steps.
First term in (\ref{r1:eqn:regret1}) is the contribution of selecting a {\em nearly} optimal bundle of contracts in exploitation steps, second term is the contribution of selecting a suboptimal bundle of contracts in the exploitation steps, and the third term is the worst-case contribution during the exploration steps to the regret.

The following theorem gives an upper bound on the regret of TLVO.
\begin{theorem}\label{thm:main1}
The regret of the seller using TLVO with time horizon $T$ is upper bounded by
\begin{align}
R(T) &\leq 5 m f_{\max} L 2^{\alpha/2} n_T^{-\alpha} (T-N(T)) + N(T) (U(\boldsymbol{x}^*) + c(n_T) - c(m)) \notag \\
&+  2 n_T \sum_{t=1}^T e^{\frac{- f^2_{\max} L^2 2^{\alpha} N(t)}{n_T^{2+ 2\alpha}}}. \notag
\end{align}
\end{theorem}
\begin{remark}\label{remark1}
\rev{In this form, the regret is linear in $n_T$ and $T$. The first term in the regret decreases with $n_T$ while the second and third terms increase with $n_T$. Since $T$ is known by the seller, $n_T$ can be optimized as a function of $T$.}
\end{remark}
\begin{proof}
Let $\delta^*_{\boldsymbol{x}} = U(\boldsymbol{x}^*) - U(\boldsymbol{x})$. By definition of the set $S$, we have
\begin{align}
\sum_{\boldsymbol{x} \in S} \delta^*_{\boldsymbol{x}} E [T_{\boldsymbol{x}}(T)] &\leq \max_{\boldsymbol{x} \in S} \delta^*_{\boldsymbol{x}} \sum_{\boldsymbol{x} \in S} E[T_{\boldsymbol{x}}(T)] \notag \\
&\leq \beta n_T^{-\alpha} (T-N(T)). \label{r1:eqn:part1}
\end{align}
%
%
%
%
%
Next, we consider the term
\begin{align*}
&\sum_{\boldsymbol{x} \in S^c} (U(\boldsymbol{x}^*) - U(\boldsymbol{x})) E \left[ T_{\boldsymbol{x}}(T) \right].
\end{align*}
Note that even if we bound $E \left[ T_{\boldsymbol{x}}(T) \right]$ for all $\boldsymbol{x} \in S^c$, in the worst case $|S^c| = c n_T^m$, for some $c>0$.
Therefore a bound that depends on $n_T^m$ will scale badly for large $m$. To overcome this difficulty, we will show that if the distribution function has sufficiently accurate sample mean estimates $\mu_i(t)$ for all
\begin{align*}
p_i :=P \left( g \left(\frac{i-1}{n_T}, \frac{i}{n_T} \right) < \theta \leq g \left(\frac{i}{n_T}, \frac{i+1}{n_T} \right)  \right), ~~ i \in \{1,2,\ldots, n_T-1 \},
\end{align*}
then the probability that some bundle in $S^c$ is offered will be small.
Let $T_{S^c}(t)$ be the number of times a bundle from $S^c$ is offered in exploitation steps by time $t$. \cm{Since for any $\boldsymbol{x} \in {\cal X}_m$, $U(\boldsymbol{x}) \leq 1$, we have}
\begin{align}
\sum_{\boldsymbol{x} \in S^c} (U(\boldsymbol{x}^*) - U(\boldsymbol{x})) E \left[ T_{\boldsymbol{x}}(T) \right] \leq  E \left[ T_{S^c}(T) \right], \label{r1:eqn:alt1}
\end{align}
where
\begin{align}
E \left[ T_{S^c}(T) \right] = E \left[ \sum_{t=1}^T I(\boldsymbol{x}(t) \in S^c)  \right] 
= \sum_{t=1}^T P(\boldsymbol{x}(t) \in S^c). \label{r1:eqn:alt2}
\end{align}
For convenience let $x_0 = 0, x_{m+1}=1$ and $g(x_m, x_m+1) =1$.
For any $x_i \in \boldsymbol{x} \in {\cal L}_{m,T}$, we can write
\begin{align*}
P \left( g(x_{i-1}, x_i) < \theta \leq g(x_i, x_{i+1}) \right) 
&= P(g(x_{i-1}, x_i) < \theta \leq j_{-}(g(x_{i-1}, x_i))) \\
&+ \sum_{i = j_{-} (g(x_{i-1},x_i))}^{ j_{+}(g(x_i,x_{i+1}))}  p_i  \\
&- P( g(x_i, x_{i+1}) < \theta \leq j_{+}(g(x_i, x_{i+1}))).
\end{align*}
Let 
\begin{align*}
err_{\boldsymbol{x}}(x_i) 
&= \left| P(g(x_{i-1}, x_i) < \theta \leq j_{-}(g(x_{i-1}, x_i))) \right. \\
& \left. - P( g(x_i, x_{i+1}) < \theta \leq j_{+}(g(x_i, x_{i+1}))) \right.|.
\end{align*}
\begin{figure}
\begin{center}
\includegraphics[width=0.8\columnwidth]{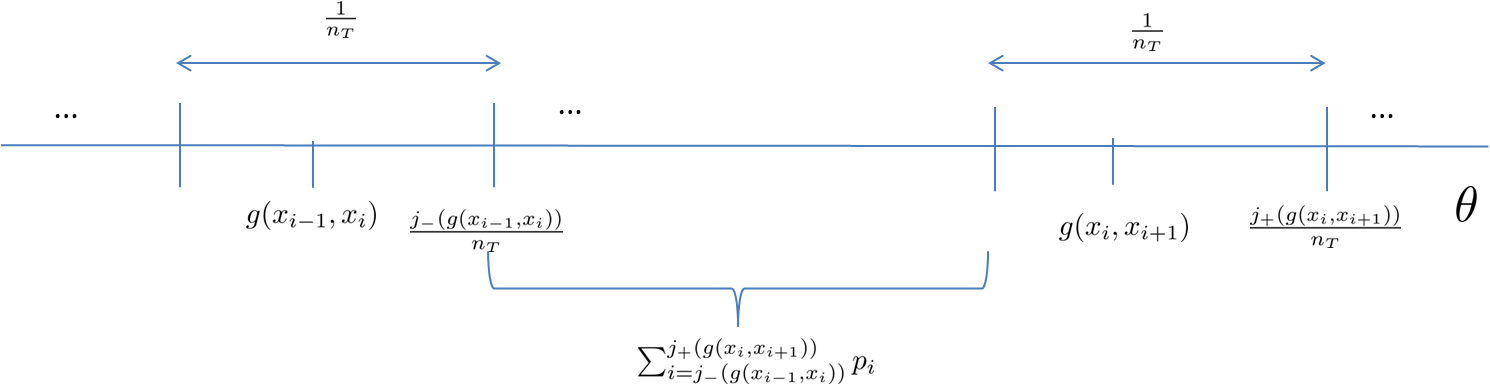}
\caption{decomposition of $P \left( g(x_{i-1}, x_i) < \theta \leq g(x_i, x_{i+1}) \right)$} 
\label{fig:decomposition}
\end{center}
\end{figure}
\cm{Figure \ref{fig:decomposition} shows the decomposition of $P \left( g(x_{i-1}, x_i) < \theta \leq g(x_i, x_{i+1}) \right)$ in terms of the acceptance regions defined by contracts in ${\cal L}_{m,T}$ and the error terms.} It is easy to see from Figure \ref{fig:decomposition} that
\begin{align*}
\left(g(x_{i-1}, x_i), \frac{j_{-}(g(x_{i-1}, x_i))}{n_T} \right] \subset \left(\frac{j_{-}(g(x_{i-1}, x_i)) -1}{n_T} , \frac{j_{-}(g(x_{i-1}, x_i))}{n_T} \right],
\end{align*}
and
\begin{align*}
\left(g(x_{i}, x_{i+1}), \frac{j_{+}(g(x_{i}, x_{i+1}))}{n_T} \right] \subset \left(\frac{j_{+}(g(x_{i}, x_{i+1})) -1}{n_T} , \frac{j_{+}(g(x_{i}, x_{i+1}))}{n_T}\right].
\end{align*}
Then, by Assumption \ref{ass:contract:ordered}, we have for any $x_i \in \boldsymbol{x} \in {\cal L}_{m,T}$
\begin{align}
& err_{\boldsymbol{x}}(x_i)  \notag \\
&\leq \max\{ P(g(x_{i-1}, x_i) < \theta \leq j_{-}(g(x_{i-1}, x_i))), P( g(x_i, x_{i+1}) < \theta \leq j_{+}(g(x_i, x_{i+1})))   \} \notag \\
&\leq f_{\max} L 2^{\alpha/2} n_T^{-\alpha}. \label{eqn:variable:error}
\end{align}

Consider the event
\begin{align}
\xi_t =  \bigcap_{i=1}^{n_T-1} \left\{ |\mu_i(t) - p_i| \leq \frac{(\beta- m f_{\max} L 2^{\alpha/2}) n_T^{-\alpha} }{4 n_T m } \right\} \notag.
\end{align}
If $\xi_t$ happens, then for any $1 \leq a < b \leq n_T -1$
\begin{align}
\left|  \sum_{i=a}^b \mu_i(t) - \sum_{i=a}^b p_i(t)   \right| &\leq (b-a) \frac{(\beta-m f_{\max} L 2^{\alpha/2})  n_T^{-\alpha} }{4 n_T m }  \notag \\
&\leq \frac{(\beta-m f_{\max} L 2^{\alpha/2}) n_T^{-\alpha} }{4 m }, \notag 
\end{align}
which implies that for any $\boldsymbol{x} \in {\cal L}_{m,T}$
\begin{align}
| \hat{U}_t(\boldsymbol{x}) - U(\boldsymbol{x})| 
 &\leq x_1   \sum_{i = j_{-} (g(0,x_1))}^{ j_{+}(g(x_1,x_2))}  \left| \mu_i(t) - p_i  \right| + err_{\boldsymbol{x}}(x_i) +  \ldots \notag \\
&+  x_m \sum_{i= j_{-} (g(x_{m-1},x_m))}^{j_{+}(g(x_{m},x_{m+1}))} \left| \mu_i(t) - p_i  \right| +  err_{\boldsymbol{x}}(x_m) \notag \\  
&\leq  \sum_{i = j_{-} (g(0,x_1))}^{ j_{+}(g(x_1,x_2))}  \left| \mu_i(t) - p_i  \right| + err_{\boldsymbol{x}}(x_i) +  \ldots \notag \\
&+  \sum_{i= j_{-} (g(x_{m-1},x_m))}^{j_{+}(g(x_{m},x_{m+1}))} \left| \mu_i(t) - p_i  \right| +  err_{\boldsymbol{x}}(x_m) \notag \\  &\leq 2m f_{\max}L 2^{\alpha/2} n_T^{-\alpha}. \label{eqn:jul9_2}
\end{align}
Let $\boldsymbol{y}^* = \argmax_{\boldsymbol{x} \in {\cal L}_{m,T}} U(\boldsymbol{x})$. By Assumption (\ref{ass:contract:ordered}) we  have
\begin{align*}
U(\boldsymbol{x}^*)-U(\boldsymbol{y}^*) \leq m f_{\max} L 2^{\alpha/2} n_T^{-\alpha}~.
\end{align*}
Then, using the definition of the set $S^c$, which denotes the set of suboptimal bundles of contracts, for any $\boldsymbol{x} \in S^c$, we have
\begin{align*}
U(\boldsymbol{y}^*) - U(\boldsymbol{x}) > (\beta - m f_{\max} L 2^{\alpha/2}) L n_T^{-\alpha} = 4m f_{\max} L 2^{\alpha/2} n_T^{-\alpha}.
\end{align*}
Since by (\ref{eqn:jul9_2}) the estimated payoff of any bundle $\boldsymbol{x} \in {\cal L}_{m,T}$ is within $2 m f_{\max} L 2^{\alpha/2} n_T^{-\alpha}$ of its true value, the event $\xi_t$ implies that for any $\boldsymbol{x} \in S^c$
\begin{align*}
\hat{U}_t(\boldsymbol{x}) \leq \hat{U}_t(\boldsymbol{y}^*),
\end{align*}
which means
\begin{align*}
\xi_t \subset \{ \hat{U}_t(\boldsymbol{x}) \leq \hat{U}_t(\boldsymbol{y}^*), \forall \boldsymbol{x} \in S^c \},
\end{align*}
and
\begin{align*}
\{ \hat{U}_t(\boldsymbol{x}) > \hat{U}_t(\boldsymbol{y}^*) \textrm{ for some } x \in S^c \} \subset \xi_t^c.
\end{align*}
Therefore
\begin{align}
P(\boldsymbol{x}(t) \in S^c) &\leq P \left( \bigcup_{i=1}^{n_T -1} \{ |\mu_i(t) - p_i| > \frac{(\beta- m f_{\max} L 2^{\alpha/2}) n_T^{-\alpha} }{4 n_T m } \} \right) \notag \\
&\leq \sum_{i=1}^{n_T - 1} P \left( |\mu_i(t) - p_i| > \frac{(\beta- m f_{\max} L 2^{\alpha/2}) n_T^{-\alpha} }{4 n_T m }  \right) \notag \\
&\leq 2 n_T e^{\frac{- f^2_{\max} L^2 2^{\alpha} N(t)}{n_T^{2+ 2\alpha}}}, \notag
\end{align}
by using a Chernoff-Hoeffding bound. Using the last result in (\ref{r1:eqn:alt1}), we get
\begin{align}
\sum_{\boldsymbol{x} \in S^c} (U(\boldsymbol{x}^*) - U(\boldsymbol{x})) E \left[ T_{\boldsymbol{x}}(T) \right] \leq 2 n_T \sum_{t=1}^T e^{\frac{- f^2_{\max} L^2 2^{\alpha} N(t)}{n_T^{2+ 2\alpha}}}. \label{r1:eqn:alt3}
\end{align}
We get the main result by substituting (\ref{r1:eqn:part1}) and (\ref{r1:eqn:alt3}) into (\ref{r1:eqn:regret1}).
\end{proof}

The following corollary gives a sublinear regret result for a special case of parameters.

\begin{corollary}\label{corr:1}
When the cost of offering $n$ contracts simultaneously, i.e., $c(n) \leq n^{\gamma}$, for all $0<n<T$, for some $\gamma >0$, the regret of the seller that runs TLVO with
\begin{align*}
n_T &= \left\lfloor (f_{\max} L 2^{\alpha/2})^{\frac{2}{4+2 \alpha}} 
\left( \frac{T}{\log T} \right)^{\frac{1}{4 + 2 \alpha}} \right\rfloor, \\
z(t) &= \left( \frac{1}{f_{\max} L 2^{\alpha/2}}    \right)^{\frac{2+6\alpha}{2+\alpha}}  \left(  \frac{T}{\log T}   \right)^{\frac{2+ 2\alpha}{4+ 2\alpha}} \log t,
\end{align*}
where $\left\lfloor y \right\rfloor$ is the largest integer smaller than equal to $y$, is upper bounded by
\begin{align*}
R(T) 
&\leq 5 m (f_{\max} L 2^{\alpha/2} )^{\frac{2}{2+\alpha}} (\log T)^{\frac{\alpha}{4+2\alpha}} T^{\frac{4+\alpha}{4+2\alpha}} \\
&+ \left( \frac{1}{f_{\max} L 2^{\alpha/2}}  \right)^{\frac{2+6\alpha}{2+\alpha}} 
(\log T)^{\frac{2+ \gamma}{4 + 2\alpha}} T^{\frac{2+2\alpha+\gamma}{4+2\alpha}} \\
&+ 2 (f_{\max} L 2^{\alpha/2} )^{\frac{1}{2+\alpha}} (\log T + 1) (\log T)^{\frac{1}{4+2\alpha}} T^{\frac{1}{4+ 2\alpha}}.
\end{align*}
Hence 
\begin{align*}
R(T) = O(m T^{(2+2\alpha+\gamma)/(4+2\alpha)} (\log T)^{2/(4+2\alpha)}),
\end{align*}
which is sublinear in $T$ for $\gamma<2$.
\end{corollary}

\begin{proof}
We want 
\begin{align*}
e^{\frac{- f^2_{\max} L^2 2^{\alpha} N(t)}{n_T^{2+ 2\alpha}}} \leq \frac{1}{t}.
\end{align*}
For this,we should have
\begin{align*}
\frac{- f^2_{\max} L^2 2^{\alpha} N(t)}{n_T^{2+ 2\alpha}} \leq - \log t,
\end{align*}
which implies
\begin{align}
N(t) \geq \frac{(n_T)^{2+2\alpha}}{f^2_{\max} L^2 2^\alpha} \log t. \label{r1:eqn:cor1}
\end{align}
Note that at each time $t$ either $N(t) \geq z(t)$ or $z(t) - 1 \leq N(t) < z(t)$ so we  chose
\begin{align}
z(t) = \frac{(n_T)^{2+2\alpha}}{f^2_{\max} L^2 2^\alpha} \log t+ 1. \notag 
\end{align}
Note that $z(t)$ in this form depends on $n_T$ which we have not fixed yet. To have minimum regret, we need to balance the first and second terms of the regret given in Theorem \ref{thm:main1}. Thus $T/n_T \approx N(T) n_T$. Since $n_T$ must be an integer, substituting (\ref{r1:eqn:cor1}) into $N(T)$, we have
\begin{align*}
n_T = \left\lfloor (f_{\max} L 2^{\alpha/2})^{\frac{2}{4+2 \alpha}} 
\left( \frac{T}{\log T} \right)^{\frac{1}{4 + 2 \alpha}} \right\rfloor.
\end{align*}
Proof is completed by substituting these into the result of Theorem \ref{thm:main1}.
\end{proof}

 \section{A Learning Algorithm with Fixed Number of Offers} \label{contract:sec:alg2}

One drawback of TLVO is that in exploration steps it simultaneously offers $n_T -1$ contracts, and this number increases sublinearly with $T$. Usually, the seller will offer different bundles of contracts but it will include same number of contracts in each bundle. For example, a wireless service provider usually adds new data plans by removing one of the current data plans, thus the total number of data plans offered does not change significantly over time. In this section, we are interested in the case when the seller is limited to offering $m$ contracts at every time step.

In this case, the exploration step of TLVO will not work. Because of this, we propose the algorithm {\em type learning with fixed number of offers} (TLFO) that always offers $m$ contracts simultaneously. TLFO differs from TLVO only in its exploration phase. Each exploration phase of TLFO lasts multiple time steps. Instead of simultaneously offering $n_T - 1$ uniformly spaced contracts at an exploration step, TLFO has an exploration phase of $\lceil (n_T-1)/(m-2) \rceil$ steps indexed by $l=1,2,\ldots, \lceil (n_T-1)/(m-2) \rceil$. The idea behind TLFO is to estimate the buyer's type distribution from the estimates of the segments of the buyer's type distribution over different time steps of the same exploration phase. 
Let time $t$ be the start of an exploration phase for TLFO. Let $l' = \lceil (n_T -1)/(m-2) \rceil$ denote the last step of the exploration phase. Next, we define the following bundles of $m$ contracts. The overlapping portions of these bundles are shown in Figure \ref{fig:allerton12} for $l=1,2,\ldots,l'$.
\begin{align*}
{\cal B}_1 &= \left\{\frac{1}{n_T}, \frac{2}{n_T}, \ldots, \frac{m}{n_T} \right\}, \\
\tilde{{\cal B}}_1 &= \left\{\frac{1}{n_T}, \frac{2}{n_T}, \ldots, \frac{m-1}{n_T} \right\}, \\
{\cal B}_{l'} &= \left\{ \frac{n_T - m}{n_T}, \frac{n_T - m + 1}{n_T}, \ldots, 
\frac{n_T - 1}{n_T} \right\}, \\
\tilde{{\cal B}}_{l'} &= \left\{ \frac{(l'-1)m - 2(l'-1)+ 2}{n_T}, \frac{(l'-1)m - 2(l'-1)+ 3}{n_T},  \ldots, \frac{n_T -1}{n_T} \right\},
\end{align*}
and for $l \in \{2,\ldots,l'-1\}$
\begin{align*}
{\cal B}_l &= \left\{ \frac{(l-1)m - 2(l-1)+ 1}{n_T}, \frac{(l-1)m - 2(l-1)+ 2}{n_T}, \ldots, 
\frac{lm - 2(l-1)}{n_T} \right\}, \\
\tilde{{\cal B}}_l &= \left\{ \frac{(l-1)m - 2(l-1)+ 2}{n_T}, \frac{(l-1)m - 2(l-1)+ 3}{n_T}, \ldots, \frac{lm - 2(l-1) -1}{n_T} \right\}.
\end{align*}
\begin{figure}
\begin{center}
\includegraphics[width=0.8\columnwidth]{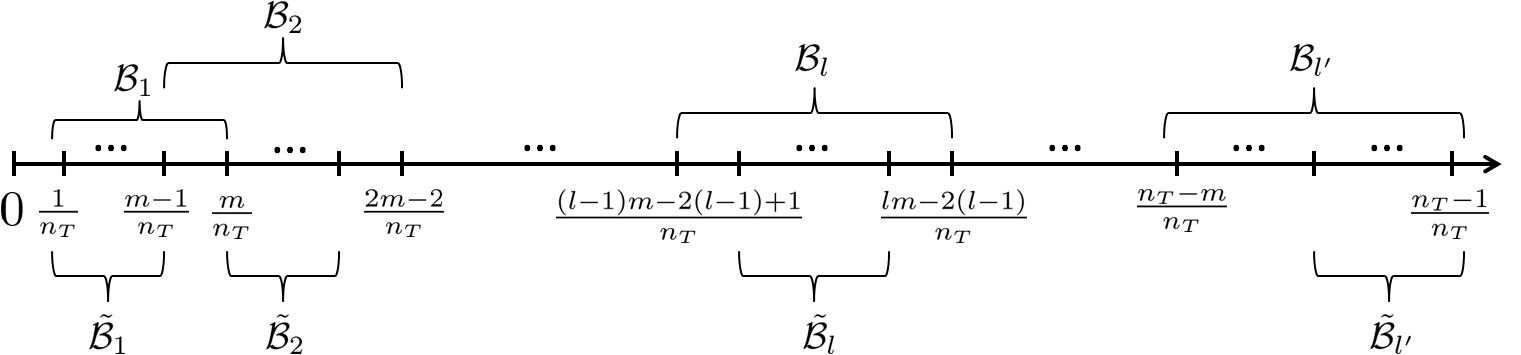}
\caption{bundles of $m$ contracts offered in exploration steps $l=1,2,\ldots, l'$ in an exploration phase} \label{fig:allerton12}
\end{center}
\end{figure}

Similar to TLVO let $N$ and $N_k$, $k \in \{1,2,\ldots, n_T-1 \}$ be the counters that are used to form type distribution estimates which are set to zero initially. Basically, at an exploitation \cm{phase} the estimates $\mu_k = N_k/N$ are formed based on the current values of $N_k$ and $N$. Different from the analysis of TLVO, $N(t)$ which is the value of counter $N$ at time $t$ represents the number of completed exploration phases by time $t$, not the number of exploration steps by time $t$.
The condition $N(t) < z(t)$ is checked at the end of each exploration phase or exploitation phase, and if the condition is true, a new exploration phase starts. In the first exploration step of the exploration phase, TLFO offers the bundle ${\cal B}_1$. If a contract $k/n_T \in \tilde{{\cal B}}_1$ is accepted, $N_k$ is incremented by one. In the $l$th exploration step, $l \in \{2,\ldots,l'-1\}$, it offers the bundle ${\cal B}_l$. If a contract $k/n_T \in \tilde{{\cal B}}_l$ is accepted, $N_k$ is incremented by one. In the last exploration step $l'$, it offers ${\cal B}_{l'}$. If a contract $k/n_T \in \tilde{{\cal B}}_{l'}$ is accepted, $N_k$ is incremented by one. At the time $t'$ when all the exploration steps in the exploration phase are completed, $N$ is incremented by one. Pseudocode of the exploration phase for TLFO is given in Figure \ref{fig:TLVOm}.

\begin{figure}[htb]
\fbox {
\begin{minipage}{0.9\columnwidth}
{Exploration phase of TLFO.}
\begin{algorithmic}[1]
\FOR{$l=1,2,\ldots, \lceil (n_T -1)/(m-2) \rceil$}
\STATE{Offer bundle ${\cal B}_l$.}
\STATE{Let $k/n_T \in {\cal B}_l$ be the accepted contract. Get reward $k/n_T$.}
\IF{$k/n_T \in \tilde{{\cal B}}_l$}
\STATE{$++N_k$}
\ENDIF
\STATE{$++t$}
\ENDFOR
\STATE{$++N$}
\end{algorithmic}
\end{minipage}
} \caption{pseudocode of the exploration phase of TLFO} \label{fig:TLVOm}
\end{figure}

Note that regret of the seller in this case is upper bounded by
\begin{align}
R(T) &\leq \sum_{\boldsymbol{x} \in S} (U_s(\boldsymbol{x}^*) - U_s(\boldsymbol{x})) E \left[ T_{\boldsymbol{x}}(T) \right] \notag \\
&+ \sum_{\boldsymbol{x} \in S^c} (U_s(\boldsymbol{x}^*) - U_s(\boldsymbol{x})) E \left[ T_{\boldsymbol{x}}(T) \right] \notag \\
&+ N(T) \lceil (n_T -1)/(m-2) \rceil (U(\boldsymbol{x}^*)). \label{r1:eqn:regretclam}
\end{align}
By the exploration phase of TLFO, the accuracy of the estimates $\mu_i(t)$ at the beginning of each exploitation \cm{phase} is the same as TLVO. Moreover, the the regret due to near-optimal exploitations can be upper bounded by the same term as in TLVO. Only the regret due to explorations changes. The number of exploration steps of TLFO is about $(n_T-1)/(m-2)$ times the number of exploration steps of TLVO, but there is no cost of offering more than $m$ (possibly a large number of) contracts in TLFO. 
The following theorem and corollary gives an upper bound on the regret of TLFO, by using an approach similar to the proofs of Theorem \ref{thm:main1} and Corollary \ref{corr:1}.
\begin{theorem}\label{thm:main2}
The regret of seller using TLFO with time horizon $T$ is upper bounded by
\begin{align}
R(T) &\leq 5 m f_{\max} L 2^{\alpha/2} n_T^{-\alpha} (T-N(T)) + N(T) \left( \frac{n_T -1}{m-2}+1 \right) U(\boldsymbol{x}^*) \notag \\
&+   2 n_T \sum_{t=1}^T e^{\frac{- f^2_{\max} L^2 2^{\alpha} N(t)}{n_T^{2+ 2\alpha}}}. \notag
\end{align}
\end{theorem}

Since TLFO simultaneously offers $m$ contracts both in explorations and exploitations, its regret does not depend on the cost function $c(.)$ of offering multiple contracts simultaneously. Therefore our sublinear regret bound always holds independent of $c(.)$.

\begin{corollary}\label{corr:2}
When the seller runs TLFO with time horizon $T$ and 
\begin{align*}
n_T &= \left\lfloor (f_{\max} L 2^{\alpha/2})^{\frac{2}{4+2 \alpha}} 
\left( \frac{T}{\log T} \right)^{\frac{1}{4 + 2 \alpha}} \right\rfloor, \\
z(t) &= \left( \frac{1}{f_{\max} L 2^{\alpha/2}}    \right)^{\frac{2+6\alpha}{2+\alpha}}  \left(  \frac{T}{\log T}   \right)^{\frac{2+ 2\alpha}{4+ 2\alpha}} \log t,
\end{align*}
we have
\begin{align*}
R(T) = C_{m} + m T^{(3+2\alpha)/(4+2\alpha)} (\log T)^{2/(4+2\alpha)},
\end{align*}
uniformly over $T$ for some constant $C_m > 0$. Hence,
\begin{align*}
R(T) = O(m T^{(3+2\alpha)/(4+2\alpha)} (\log T)^{2/(4+2\alpha)}).
\end{align*}

\end{corollary}

 \section{Discussion}\label{contract:sec:discuss}

A contract design problem for a secondary spectrum market is studied in \cite{shangpin2012contract}. In this work the authors assume that the type distribution $f(\theta)$ is known by the seller, and they characterize the optimal set of contracts. They show that when the channel condition is common to all types, i.e., probability that the channel is idle is the same for all types of users, a computationally efficient procedure exists for choosing the best bundle of $m$ contracts out of ${\cal L}_{m,T}$. This procedure can be used by the seller to efficiently solve (\ref{alg1:eqn:1}).


In the fixed number of offers case, we assume that at each time step the seller offers a bundle $(x_1, x_2, \ldots, x_m) \subset {\cal X}_m \subset [0,1]^m$. Therefore, the strategy set is a subset of the $m$-dimensional unit cube. Because of this relation, we can compare the performance of our contract learning algorithms with bandit algorithms for high dimensional strategy sets. For example, if the reward from any bundle $\boldsymbol{x}$ were of linear form, i.e., $U(\boldsymbol{x}) = \boldsymbol{C} \cdot \boldsymbol{x}$ for some $\boldsymbol{C} \in \mathbb{R}^m$, then the online stochastic linear optimization algorithm in \cite{dani2008stochastic} would give regret $O((m \log T)^{3/2} \sqrt{T})$. However, in our problem $U(\boldsymbol{x})$ is not a linear function, thus this approach will not work. One can also show that in general $U(\boldsymbol{x})$ is neither convex or nor concave, therefore any bandit algorithm exploiting these properties will not work in our setting.

Another work, \cite{bubeck:inria-00329797}, considers online linear optimization in a general topological space. For an $m$-dimensional strategy space, they prove a lower bound of $\tilde{O}(T^{(m+1)/(m+2)})$. Therefore, our bound is better than their lower bound for $m>2+2\alpha$. This is not a contradiction since in our problem it is the type $\theta$ that is drawn independently at each time step, not the rewards of the individual contracts, and we focus on estimating the expected rewards of arms (bundles of contracts) from the type distribution. In the same paper, a $\tilde{O}(\sqrt{T})$ regret upper bound is also proved, under the assumption that the mean reward function is locally equivalent to a bi-H\"{o}lder function near any maxima, i.e., $\exists c_1, c_2, \epsilon_0 > 0$ such that for $||\boldsymbol{x} - \boldsymbol{x'}|| \leq \epsilon_0$
\begin{align*}
c_1 ||\boldsymbol{x} - \boldsymbol{x}'||^{\alpha} \leq |U(\boldsymbol{x}) - U(\boldsymbol{x}')| \leq c_2 ||\boldsymbol{x} - \boldsymbol{x}'||^{\alpha}.
\end{align*}
However, in this paper, we only require a H\"{o}lder condition for the boundaries of the acceptance regions (see Assumption \ref{ass:contract:ordered}), which implies that
\begin{align*}
|U(\boldsymbol{x} ) - U(\boldsymbol{x}')| \leq c_3 ||\boldsymbol{x} - \boldsymbol{x}'||^{\alpha},
\end{align*}
for some $c_3 > 0$ and $\forall \boldsymbol{x}, \boldsymbol{x}' \in {\cal X}_m$.


\bibliographystyle{IEEE}
\bibliography{cem}


\end{document}